\documentclass[runningheads]{llncs}
\usepackage{graphicx}
\usepackage{amsfonts}
\usepackage{float}
\usepackage{physics}
\usepackage{tabularx}
    \newcolumntype{Y}{>{\centering\arraybackslash}X}
\usepackage{booktabs}
\usepackage{pdflscape}
\usepackage{csvsimple}
\usepackage{hyperref}[hyperfootnotes=true]
\usepackage{url}
\usepackage{nameref}
\usepackage{tablefootnote}
\usepackage{color, colortbl}
\usepackage{mathtools}
\usepackage[algo2e, ruled, vlined, linesnumbered]{algorithm2e}
\usepackage{xparse}
\usepackage{algorithm2e}
\usepackage{algorithm}
\usepackage{makecell,booktabs}

\usepackage
[
    noabbrev,   
    nameinlink, 
]
{cleveref} 

\DeclareDocumentCommand{\set}{m g o}{
    \ensuremath{
        \IfNoValueTF{#3}{\left}{#3}\{#1
            \IfNoValueTF{#2}{}{
                \ \IfNoValueTF{#3}{\left}{#3}\vert\ \vphantom{#1}#2\IfNoValueTF{#3}{\right.}{}
            } \IfNoValueTF{#3}{\right}{#3}\}
    }\xspace
}

\definecolor{lightgreen}{rgb}{0.75,0.92,0.61}

\newcommand{\labelname}[1]{
  \def\@currentlabelname{#1}}

\providecommand{\ignore}[1]{}

\newcommand{\oneoneEA}{(1+1) EA\xspace}
\newcommand{\muoea}{($\mu$+1) EA\xspace}
\newcommand{\cGA}{cGA\xspace}

\newcommand{\om}{\textsc{OM}_{D}\xspace}
\newcommand{\OneMax}{\textsc{OneMax}\xspace}
\newcommand{\LeadingOnes}{\textsc{LeadingOnes}\xspace}

\DeclareMathOperator{\Geo}{Geo}

\DeclareMathOperator{\erfc}{erfc}

\newcommand{\normal}[1]{\mathcal{N}(#1)}
\newcommand{\eqnComment}[2]{\underset{{\scriptstyle \text{#1}}}{#2}}

\newcommand{\realnum}{\mathbb{R}}
\newcommand{\natnum}{\mathbb{N}}

\newtheorem{thm}{Theorem}
\newtheorem{lem}[thm]{Lemma}

\newenvironment{myAlgorithm}%
	{
	\begin{center}
	\begin{algorithm2e}
	}%
	{
	\end{algorithm2e}
	\end{center}
	}
	
\allowdisplaybreaks

\begin{document}
\title{Theoretical Study of Optimizing Rugged Landscapes with the cGA}
%
%
\author{Tobias Friedrich\inst{1}
\and
Timo Kötzing\inst{1}
\and
Frank Neumann\inst{2}
\and
Aishwarya Radhakrishnan\inst{1}}
\authorrunning{Friedrich et al.}
%

\institute{Hasso Plattner Institute, University of Potsdam, Potsdam, Germany
\email{friedrich@hpi.de, timo.koetzing@hpi.de, aishwarya.radhakrishnan@hpi.de}\\
\and
School of Computer Science, The University of Adelaide, Adelaide, Australia\\
\email{frank.neumann@adelaide.edu.au}}
\maketitle              
\begin{abstract}
Estimation of distribution algorithms (EDAs) provide a dis\-tri\-bution-based approach for optimization which adapts its probability distribution during the run of the algorithm. We contribute to the theoretical understanding of EDAs and point out that their distribution approach makes them more suitable to deal with rugged fitness landscapes than classical local search algorithms. 

Concretely, we make the \OneMax function rugged by adding noise to each fitness value. The cGA can nevertheless find solutions with $n(1-\varepsilon)$ many $1$s, even for high variance of noise. In contrast to this, RLS and the (1+1) EA, with high probability, only find solutions with $n(1/2+o(1))$ many $1$s, even for noise with small variance.

\end{abstract}
\section{Introduction}
\label{sec:intro}

Local search~\cite{10.5555/549160}, evolutionary algorithms~\cite{DBLP:series/ncs/EibenS15} and other types of search heuristics have found applications in solving classical combinatorial optimization problems as well as challenging real-world optimization problems arising in areas such as mine planning and scheduling~\cite{DBLP:conf/gecco/MyburghD10,DBLP:conf/cec/OsadaWBM13} and renewable energy~\cite{DBLP:conf/gecco/TranWDA0N13,DBLP:journals/isci/NeshatAW20}.

Local search techniques perform well if the algorithm can achieve improvement through local steps whereas other more complex approaches such as evolutionary algorithms evolving a set of search points deal with potential local optima by diversifying their search and allowing to change the current solutions through operators such as mutation and crossover.
Other types of search heuristics, such as estimation of distribution algorithms~\cite{PelikanHandbook15} and ant colony optimization~\cite{DBLP:books/daglib/0013523}, sample their solutions in each iteration from a probability distribution that is adapted based on the experience made during the run of the algorithm. One of the key questions that arises is when to favour distribution-based algorithms over search point-based methods. We will investigate this in the context of rugged landscapes that are obtained by stochastic perturbation.

Real-world optimization problems are often rugged with many local optima and quantifying and handling rugged landscapes is an important topic when using search heuristics~\cite{DBLP:conf/cec/MalanE09,poursoltan2015ruggedness,10.1162/evco_a_00193}. Small details about the chosen search point can lead to a rugged fitness landscapes even if the underlying problem has a clear fitness structure which, by itself, would allow local search techniques to find high quality solution very quickly.

In this paper we model a rugged landscape with underlying fitness structure via \OneMax,\footnote{\OneMax is the well-studied pseudo-Boolean test function mapping $x \in \{0,1\}^n$ to $\sum_{i=1}^n x_i$.} where each search point is perturbed by adding an independent sample from some given distribution $D$. We denote the resulting (random) fitness function $\om$.

Note that this setting of $D$-rugged \OneMax is different from noisy \OneMax (with so-called additive posterior noise) \cite{GiessenK16,timo,paper2,poster,compact} in that evaluating a search point multiple times does not lead to different fitness values (which then could be averaged, implicitly or explicitly) to get a clearer view of the underlying fitness signal. Note that another setting without reevaluation (but on a combinatorial path problem and for an ant colony algorithm, with underlying non-noisy ground truth) was analyzed in \cite{paper3}.

In this paper we consider as distribution the normal distribution as well as the geometric distribution. Since all search points get a sample from the same distribution added, the mean value of the distribution is of no importance (no algorithm we consider makes use of the absolute fitness value, only of relative values). Mostly important is the steepness of the tail, and in this respect the two distributions are very similar.

An important related work, \cite{FriedrichKKS16} discusses what impact the \emph{shape} of the chosen noise model has on optimization of noisy \OneMax. They find that steep tails behave very differently from uniform tails, so for this first study we focus on two distributions with steep tails, which we find to behave similarly.

As a first algorithm, which was found to perform very well under noise \cite{poster}, we consider the \emph{compact genetic algorithm} (cGA), an estimation of distribution algorithm which has been subject to a wide range of studies in the theoretical analysis of estimation of distribution algorithms \cite{compact,DoerrZ20a,Doerr21,LenglerSW21}. See Section~\ref{sec:algoAndProblem} for an exposition of the algorithm.

In Theorem~\ref{thm:cGAOnRuggedOneMax} we show the cGA to be efficient on $\normal{0,\sigma^2}$-rugged \OneMax, even for arbitrarily large values of $\sigma^2$ (at the cost of a larger run time bound). Note that, since the optimum is no longer guaranteed to be at the all-1s string (and the global optimum will be just a rather random search point with a lot of $1$s, we only consider the time until the cGA reaches $n(1-\varepsilon)$ many $1$s (similar to \cite{paper2}).
The idea of the proof is to show that, with sufficiently high probability, no search point is evaluated twice in a run of the cGA; then the setting is identical to \OneMax with additive posterior noise and we can cite the corresponding theorem from the literature \cite{compact}. Thus, working with a distribution over the search space and adapting it during the search process leads to a less coarse-grained optimization process and presents a way to deal effectively with rugged fitness landscapes.

We contrast this positive result with negative results on search point based methods, namely random local search (RLS, maintaining a single individual and flipping a single bit each iteration, discarding the change if it worsened the fitness) and the so-called \oneoneEA (which operates like RLS, but instead of flipping a single bit, each bit is flipped with probability $1/n$), as well as with Random Search (RS, choosing a uniformly random bit string each iteration).

We first consider random local search on $\normal{0,\sigma^2}$-rugged \OneMax. Theorem~\ref{thm:RLSonNormal} shows that, for noise in $\Omega(\sqrt{\log n})$, RLS will not make it further than half way from the random start to the \OneMax-optimum and instead get stuck in a local optimum. The proof computes, in a rather lengthy technical lemma, the probability that a given search point has higher fitness than all of its neighbors.

Going a bit more into detail about what happens during the search, we consider the geometric distribution. In Theorem~\ref{rls} we prove that, even for constant variance, with high probability the algorithm is stuck after transitioning to a new search point at most $\log^2 (n)$ many times. This means that essentially no progress is made over the initial random search point, even for small noise values! The proof proceeds by showing that successive accepted search points have higher and higher fitness values; in fact, the fitness values grow faster than the number of $1$s in the underlying bit string. Since every new search point has to have higher fitness than the previous, it quickly is unfeasible to find even better search points (without going significantly higher in the $1$s of the bit string).

In Theorem~\ref{thm:OneOneEAonGeometric} we translate the result for RLS to the \oneoneEA. We require small but non-constant variance of $D$ to make up for the possibility of the \oneoneEA to jump further, but otherwise get the result with an analogous proof. To round off these findings, we show in Theorem~\ref{thm:RSonOM} that Random Search has a bound of $O(\sqrt{n\log n})$ for the number of $1$s found within a polynomial number of iterations.

In Section~\ref{sec:ExperimentsonOneMax} we give an experimental impression of the negative results. We depict that, within $n^2$ iterations, the proportion of $1$s in the best string found is decreasing in $n$ for all algorithms RLS, \oneoneEA and RS, where RS is significantly better than the \oneoneEA, RLS being the worst. 

This paper proceeds with some preliminaries on the technical details regarding the algorithms and problems considered. After the performance analyses of the different algorithms in Section~\ref{sec:cGAonOneMax} through~\ref{sec:randomSearch} and the experimental evaluation in Section~\ref{sec:ExperimentsonOneMax}, we conclude in Section~\ref{sec:conclusion}. We defer some technical lemmas to Appendix~\ref{sec:lemmas}.

\ignore{
Furthermore, we show that the distribution based approach based on the cGA is also able to outperform state of the art algorithms for noisy real-world problems. We demonstrate this for the well-known influence maximization problem in social networks, where the propagation of influence is stochastic. Our results show that the cGA significantly outperforms state or the art Pareto optimization approaches for this problem.

\subsection{Related Work}
 This paper \cite{timo} shows that a mutation-only Evolutionary Algorithm (\muoea) does not optimize noisy \OneMax in polynomial time when the variance is high. But cGA, a simple estimation of distribution algorithm which uses recombination can always optimize the \OneMax function with noise in polynomial time, if the noise variance $\sigma^2$ is bounded by some polynomial in $n$. 

 The work here \cite{paper2} shows that paired-crossover evolutionary algorithm(PCEA) with a population of $\Omega(\sqrt{n}$ log($n$)) solves \OneMax with noise in an expected run time of $O(\sqrt{n}$ log($n$)) with high probability. 
 
 The experimental study here \cite{poster} on a set of evolutionary algorithms on different noisy problems says that PCEA and UMDA handles noise well compared to the other algorithms considered.
 
 This \cite{paper} emphasis on importance of population on expected number of iterations while optimizing noisy versions of \OneMax and \LeadingOnes. It starts with giving an upper bound and a lower bound for ($1+1$)EA on \OneMax followed by a run time analysis for ($1+1$)EA on \OneMax in presence of different types of prior and posterior noises. Then upper bounds for EAs using parent populations (($\mu+1$)EA) and offspring populations(($1+\lambda$)EA) on noisy \OneMax are given. The same was studied on noisy \LeadingOnes also and both together lead to a surprising result that even very small population can lead to very high robustness to noise. 

 The paper here \cite{paper3} presents a first analysis of Ant Colony Optimization(ACO) for a stochastic shortest path problem (MMAS\textsubscript{SDSP}). A general upper bound until a approximation is found in presence of independent noise is explained followed by a lower bound in presence of independent gamma-distributed noise. A situation where the algorithm could not find an approximation even in exponential time is presented and it concludes by saying that the algorithm works efficiently when the noise is perfectly correlated.
}

\section{Algorithms and Problem Setting}
\label{sec:algoAndProblem}

In this section we define the $D$-rugged \OneMax problem and describe all the algorithms which we are analyzing in this paper. Random local search (RLS) on a fitness function $f$ is given in Algorithm \ref{alg:RLS}. RLS samples a point uniformly at random and at each step creates an offspring by randomly flipping a bit. At the end of each iteration it retains the best bit string available.
\vspace*{-.4cm}
\begin{myAlgorithm}
	Choose $x \in \{0,1\}^n$ uniformly at random\;\label{step:initialization_RLS}
	\While{stopping criterion not met}{
	$y$ $\gets$ flip one bit of $x$ chosen uniformly at random\;
	\lIf{$f(y) \geq f(x)$}{$x \gets y$\label{step:apdopt_current_best_RLS}}
	}
	\caption{RLS on fitness function $f$}\label{alg:RLS}
\end{myAlgorithm}
\vspace*{-.4cm}
The \oneoneEA on a fitness function $f$ is given in Algorithm \ref{alg:1+1EA}. The difference between RLS and \oneoneEA is that \oneoneEA creates an offspring by flipping each bit with probability $1/n$.
\vspace*{-.4cm}
\begin{myAlgorithm}
	Choose $x \in \{0,1\}^n$ uniformly at random\;\label{step:initialization_EA}
	\While{stopping criterion not met}{
	$y$ $\gets$ flip each bit of $x$ independently with probability $1/n$\;
	\lIf{$f(y) \geq f(x)$}{$x \gets y$\label{step:apdopt_current_best_EA}}
	}
	\caption{(1+1) EA on fitness function $f$}\label{alg:1+1EA}
\end{myAlgorithm}
\vspace*{-.4cm}
The \cGA on a fitness function $f$ is given in Algorithm \ref{alg:cGA}. This algorithm starts with two bit strings which have the probability of $1/2$ for each of their bit to be $1$. After each step this probability is updated based on the best bit string encountered.
\vspace*{-.4cm}
\begin{myAlgorithm}
    $t$ $\gets$ $0$, K $\gets$ initialize\;
    $p_{1,t}$ $\gets$ $p_{2,t}$ $\gets$ $\cdots$ $\gets$ $p_{n,t}$ $\gets$ $1/2$\;
    \While{termination criterion not met}{
        \For{$i \in \{1,\dotsc,n\}$}{
            $x_i \gets 1$ with probability $p_{i,t}$, $x_i \gets 0$ else\;
        }
        \For{$i \in \{1,\dotsc,n\}$}{
            $y_i \gets 1$ with probability $p_{i,t}$, $y_i \gets 0$ else\;
        }
        \lIf{$f(x) < f(y)$}{swap $x$ and $y$}
        \For{$i \in \{1,\dotsc,n\}$}{
        \lIf{$x_i > y_i$}{$p_{i,t+1} \gets p_{i,t} + 1/K$}
        \lIf{$x_i < y_i$}{$p_{i,t+1} \gets p_{i,t} - 1/K$}
        \lIf{$x_i = y_i$}{$p_{i,t+1} \gets p_{i,t}$}
        }        
        $t \gets t+1$\;
    }
	\caption{The compact GA on fitness function $f$}\label{alg:cGA}

\end{myAlgorithm}
\vspace*{-.4cm}

\subsection{$D$-Rugged OneMax}

To give a simple model for a rugged landscape with underlying gradient, we use a randomly perturbed version of the well-studied OneMax test function. We fix a dimension $n \in \natnum$ and a random distribution $D$. 
Then we choose, for every  $x \in \{0,1\}^n$, a random distortion $y_x$ from the distribution $D$. We define a $D$-rugged OneMax function as
$
\om \colon \{0,1\}^n \to \mathbb{R} := x \mapsto \norm{x}_1 + y_x$
where $\norm{x}_1 := |\set{i}{ x_i =1 }|$ is the number of $1$s in $x$.

In the following sections we show that the \cGA optimizes even very rugged distortions of $\om$ efficiently, while RLS will get stuck in a local optimum. 


\section{Performance of the \cGA}\label{sec:cGAonOneMax}

Let $D \sim N(0, \sigma^2)$. The following is Lemma 5 from \cite{compact}, which shows that, while optimizing $\om$, the probability that marginal probabilities falls a constant less than 1/2 is superpolynomially small in $n$. Note that here and in all other places in this paper, we give probabilities that range over the random choices of the algorithm as well as the random landscape of the instance.

\begin{lem}\label{lem:movingWeight}
Let $\varepsilon \in (0,1)$ and define
$$
M_\varepsilon = \set{p \in [0.25,1]^n}{ \sum_{i=1}^n p_i \leq n(1-\varepsilon)}.
$$
Then
$$
\max_{p,q \in M_\varepsilon} \sum_{i=1}^n (2p_iq_i-p_i-q_i) \leq - n \varepsilon/2.
$$
\end{lem}

\begin{lem}[{\cite{compact},\cite{genetic_drift}}]\label{lem:lowerbound}
Let $D \sim N(0, \sigma^2)$. Consider the \cGA optimizing $\om$ with $\sigma^2 > 0$. Let $0 < a < 1/2$ be an arbitrary constant and $T' = \min\{t \geq 0 : \exists i \in [n], p_{i,t} \leq a\}$. If $K = \omega(\sigma^2 \sqrt{n}\log n)$, then for every polynomial $\mathrm{poly}(n)$, n sufficiently large, $ \Pr(T' < \mathrm{poly}(n))$ is superpolynomially small.
\end{lem}

\begin{thm}\label{thm:noDuplicates}
Let $\varepsilon \in (0,1)$ and define
$$
M_\varepsilon = \set{p \in [0.25,1]^n}{ \sum_{i=1}^n p_i \leq n(1-\varepsilon)}.
$$
Let $S \subseteq \{0,1\}^n$ be a random multi set of polynomial size (in $n$), where each member is drawn independently according to some marginal probabilities from $M_\varepsilon$. Then the probability that the multi set contains two identical bit strings is $2^{-\Omega(n\varepsilon)}$.
\end{thm}
\begin{proof}
Let $x,y \in S$ be given, based on marginal probabilities $p,q \in M_\varepsilon$, respectively. Using Lemma~\ref{lem:movingWeight} in the last step, we compute
\begin{align*}
P(x=y) &= \prod_{i=1}^n P(x_i=y_i)\\
 &= \prod_{i=1}^n (p_iq_i + (1-p_i)(1-q_i))\\
 &= \prod_{i=1}^n (1 - p_i - q_i + 2p_iq_i)\\
 &\leq \prod_{i=1}^n \exp(- p_i - q_i + 2p_iq_i)\\
 &\leq  \exp( \sum_{i=1}^n - p_i - q_i + 2p_iq_i)\\
 &\eqnComment{Lemma~\ref{lem:movingWeight}}{\leq} \exp( -n \varepsilon/2). 
\end{align*}
Since there are only polynomially many pairs of elements from $S$, we get the claim by an application of the union bound.
\end{proof}
\qed

From the preceding theorem we can now assume that the \cGA always samples previously unseen search points. We can use \cite[Lemma~4]{compact} which gives an additive drift of $O(\sqrt{n}/(K\sigma^2))$ in our setting. 

Now we can put all ingredients together and show the main theorem of this section.

\begin{thm}\label{thm:cGAOnRuggedOneMax}
Let $D \sim N(0, \sigma^2)$, $\sigma^2 > 0$ and let $\varepsilon$ be some constant. Then the \cGA with $K = \omega(\sigma^2 \sqrt{n}\log n)$ optimizes $\om$ up to a fitness of $n(1 - \varepsilon)$ within an expected number of $O(K\sqrt{n}\sigma^2)$ iterations.
\end{thm}
\begin{proof}
We let $X_t = \sum_{i=1}^n p_{i,t}$ be the sum of all the marginal probabilities at time step $t$ of the \cGA. Using Lemma~\ref{lem:lowerbound}, we can assume that (for polynomially many time steps) the \cGA has marginal probabilities of at least $0.25$. Now we can employ Lemma~\ref{thm:noDuplicates} to see that the \cGA does not sample the same search point twice in a polynomial number of steps. Thus, as mentioned, we can use \cite[Lemma~4]{compact} to get an additive drift of $O(\sqrt{n}/(K\sigma^2))$ as long as $X_t$ does not reach a value of $n(1-\varepsilon)$.

The maximal value of $X_t$ is $n$, so we can use an additive drift theorem that allows for overshooting \cite[Theorem~3.7]{MartinKDissertation} to show the claim.
\end{proof}
\qed

\section{Performance of RLS}\label{sec:RLSonOneMax}

In this section we show that RLS cannot optimize rugged landscapes even for small values of ruggedness. We show that RLS will not find a solution with more than $3n/4$ ones with high probability. This implies that RLS will get stuck in a local optimum with a high probability because there are exponentially many points with number of ones more than $3n/4$ and the probability that none of this points have noise more than the noise associated with the best solution found by RLS is very low.

\begin{thm}\label{thm:RLSonNormal}
Let $D \sim N(0, \sigma^2)$. Let $\sigma^2 \geq 4\sqrt{2\ln(n+1)}$. Then there is a constant $c < 1$ such that RLS optimizing $\om$ will reach a solution of more than $3n/4$ many $1$s (within any finite time) will have a probability of at most $c$.
\end{thm}
\begin{proof}
We consider the event $A_0$ that RLS successfully reached a fitness of $3n/4$ starting from an individual with at most $n/2$ many $1$s. 

With probability at least $1/2$ the initial search point of $RLS$ has at most $n/2$ many $1$s.

We define, for each level $i \in \{n/2, n/2 + 3, n/2+6,\ldots, 3n/4 - 3\}$, the first \emph{accepted} individual $x_i$ which RLS found on that level. For the event $A_0$ to hold, it must be the case that all $x_i$ are \emph{not} local optima. Any search points in the neighborhood of $x_i$ sampled previous to the encounter of $x_i$ will have a value less than $x_i$ (since $x_i$ is accepted) and the decision of whether $x_i$ is a local optimum depends only on the $k < n$ so far not sampled neighbors. Since two different $x_i$ have a Hamming distance of at least $3$, these neighbors are disjoint sets (for different $i$) and their noises are independent.

For any point $x$ to be a local optimum, it needs to have a higher fitness than any of its at most $n$ neighbors. We assume pessimistically that all neighbors have one more $1$ in the bit string and compute a lower bound on the probability that the random fitness of $x$ is higher than the random fitness of any neighbor by bounding the probability that a Gaussian random variable is larger than the largest of $n$ Gaussian random variables plus $1$. By scaling, this is the probability that some $\mathcal{N}(0, 1)$-distributed random variable is higher than the maximum of $n$ independent $\mathcal{N}(1/\sigma^2, 1)$-distributed random variables.

Using the symmetry of the normal distribution, this is equivalent to the probability that some $\mathcal{N}(1/\sigma^2, 1)$-distributed random variable is \emph{less} than the \emph{minimum} of $n$ $\mathcal{N}(0, 1)$-distributed random variables. This is exactly the setting of Lemma~\ref{lem:minimumOfGaussians}, where we pick $c = 1/\sigma^2$. Plugging in our bound for $\sigma^2$, we get a probability of $\Omega(1/n)$ that an arbitrary point in our landscape is a local optimum.

Thus we get with constant probability that one of the $\Theta(n)$ many $x_i$ is a local optimum. With this constant probability, event $A_0$ cannot occur, as desired.
\end{proof}
\qed

\subsection{Performance of RLS -- a detailed look}

We now want to give a tighter analysis of RLS on rugged \OneMax by showing that, in expectation, the noise of new accepted search points is growing. For the analysis, we will switch to a different noise model: We now assume our noise to be Geo$(p)$-distributed, for some $p \leq 1/2$. We believe that a similar analysis is also possible for normal-distributed noise, but in particular with much more complicated Chernoff bounds.

\begin{thm}\label{rls}
Let $p \leq 1/2$ and let $D \sim \Geo(p)$.
Then, for all $c$, the probability that RLS optimizing $D$-rugged \OneMax will transition to a better search point more than $log^2 (n)$ times is $O(n^{-c})$.

In particular, in this case, RLS does not make progress of $\Omega(n)$ over the initial solution and does not find the optimum.
\end{thm}
\begin{proof}
We consider the run of RLS optimizing $D$-rugged \OneMax and denote with $X_t$ the \emph{noise} of the $t$-th \emph{accepted} search point. We know that $X_0 \sim D$ and each next point has to be larger than at least the previous search point minus 1: in each iteration either a 0 bit or a 1 bit is flipped and RLS accepts the new search point only if its fitness value is greater than or equal to the previous search point.  

We will show that $X_t$ is, in expectation, growing. Furthermore we will show that, with high probability, for $t= \log^2(n)$ we have that $X_t \geq \log^2(n) / 2$. We finish the proof by showing that, with a search point with such a noise value, it is very unlikely to find a better search point.

Note that, for all $t$, we have that the distribution of $X_{t+1}$ is the distribution of $D$ conditional on being at least $X_t-1$ if a $0$ was flipped to a $1$ and $X_t+1$ otherwise. Pessimistically assuming the first case and since $D$ is memory-less, we get $X_{t+1} \sim X_t-1 + D$. In particular, since $E[D] = \frac{1}{p}$, we have $E[X_{t+1}] \geq E[X_t] + 1/p - 1$. 
Inductively we get
\begin{equation}
E[X_t] \geq E[X_0] + \frac{t}{p} - t.    
\end{equation}

Let the geometric random variable attached with each $X_t$ be $D_t$ and we have $X_{t} \sim X_{t-1}-1 + D_t $, therefore $X_t \sim \sum_{i = 0}^{t} D_i - t$. This implies $P(X_t \leq (\frac{t+1}{p}-t)/2)$ is nothing but $P( \sum_{i = 0}^{t} D_i \leq (\frac{t+1}{p}+t)/2)$. By using Chernoff bounds for the sum of independent geometric random variables \cite[Theorem~1.10.32]{Doerr_2019} and by letting $\delta = \frac{1}{2} - \frac{tp}{2(t+1)}$ we have, 
\begin{align*}
P\left(X_t \leq \left(\frac{t+1}{p} - t\right)/2\right) &=  P\left( \sum_{i = 0}^{t} D_i \leq \left(\frac{t+1}{p}+t\right)/2\right) \\
&= P\left( \sum_{i = 0}^{t} D_i \leq (1-\delta)\frac{t+1}{p}\right) \\
&\leq \exp\left(-\frac{\delta^2(t+1)}{2 - \frac{4\delta}{3}}\right).
\end{align*}
Since $p \leq \frac{1}{2}$, we have $\delta \geq \frac{1}{4}$. Therefore,
\[P\left(X_t \leq \left(\frac{t+1}{p} - t\right)/2\right) \leq \exp\left(-\frac{\delta^2(t+1)}{2 - \frac{4\delta}{3}}\right) \leq \exp\left(-\frac{3t}{80}\right)\]
When $t = \log^2 (n)$, 
\[P(X_t \leq \left(\frac{t+1}{p} - t\right)/2) \leq \exp\left(-\frac{3t}{80}\right) = n^{-\frac{3}{80}\log(n)}.\]
Assume that we sampled a search point with noise at least $ m = \frac{ t + 1}{2p} - \frac{t}{2}$, where $t = \log^2 (n)$. 
For a neighbor of the current search point to have higher fitness it should have at least $m - 1$ or $m + 1$ noise, depending on whether it has an extra 1 bit or an extra zero bit. The probability for this to happen is,
\begin{align*}
P(D \geq m + 1) &\leq P(D \geq m - 1)\\
    &= p(1-p)^{-2}(1-p)^{m}    \\
    &\leq e^{\frac{1}{2}}n^{-\frac{1}{4}\log(n)}.
\end{align*}
Using this, we will show that once a search point with noise at least $m$ is sampled, the probability that at least one of the neighbours is of higher fitness is $O\left(n^{1-\frac{\log(n)}{4}}\right)$. Let $D_{m_{1}}, \dots , D_{m_{n}}$ denote the random geometric noise associated with the neighbors of the current search point with at least noise $m$. Then probability that at least one of the neighbours is of higher fitness is 
\begin{align*}
    &\leq P(D_{m_{1}} \geq m - 1 \cup \dots \cup D_{m_{n}} \geq m - 1)\\
    &\leq \sum_{i=1}^{n}  P(D_{m_{i}} \geq m - 1)\\
    & \leq e^{\frac{1}{2}}n^{1-\frac{\log(n)}{4}}.
\end{align*}


\end{proof}
\qed

\section{Performance of the (1+1) EA}

In this section we extend the analysis given for RLS in Theorem~\ref{rls} to the \oneoneEA.

\begin{thm}\label{thm:OneOneEAonGeometric}
Let $p \leq 1/(2\log(n))$ and let $D \sim \Geo(p)$.
Then, for all $c > 0$ and $k \in \natnum$, the probability that the \oneoneEA optimizing $D$-rugged \OneMax will transition to a better search point more than $\log^2 (n)$ times within $n^k$ steps is $O(n^{-c})$.

In particular, the probability that the \oneoneEA makes progress of $\Omega(n)$ over the initial solution within $n^k$ steps is $O(n^{-c})$ and thus does not find the optimum.
\end{thm}
\begin{proof}
We first show that, for $c > 0$, in any iteration  the new accepted search point by \oneoneEA does not have more than $ck\log(n) - 1$ ones than the previous accepted point with probability at least $1-O(n^{-c})$. Then we assume the worst case scenario that the new search point has $ck\log(n) - 1$ more ones than the previous search point to proceed with the proof similar to Theorem \ref{rls} for RLS.

Let $X$ denote the number of bit flips happened to get the current search point. Then $X \sim \text{Bin}(n,1/n)$ (which has an expectation of $1$). If the current search point has $ck\log(n)-1$ more ones than the previous search point then $X$ has to be at least $ck\log(n) - 1$. By a multiplicative Chernoff bound for the sum of independent Bernoulli trails, if $\delta = ck\log(n) - 2$,
\begin{align*}
P(X \geq ck\log(n) - 1) 
    &= P(X \geq 1 + \delta)\\
    &\leq \exp\left(- \frac{(ck\log(n) -2)^2}{ck\log(n)}\right)\\
    &\leq \frac{e^4}{n^{ck}}.
\end{align*}
A union bound over $n^k$ iterations gives that the probability that within any of these iterations the \oneoneEA jumps further than $ck\log(n) -1$ is $O(n^{-c})$.

Consider now the run of the \oneoneEA optimizing $D-$rugged \OneMax and let $X_t$ denote the noise associated with the $t-$th accepted search point. Similar to Theorem~\ref{rls}, we now show the following. (1) $X_t$ grows in expectation. (2) For $t = \log^2 (n)$, with high probability $X_t \geq \log^2 (n)/2$. (3) If we have a search point with noise greater than $\log^2 (n)/2$, the probability to find a better search point within $n^k$ steps is very low.
We pessimistically assume that the $t$-th accepted search point has $\log(n) - 1$ ones more than the previous search point. Since $D$ is memory-less, we get $X_t \sim X_{t-1} - \log(n) + 1 + D$. This inductively along with the fact that $X_0 = D$ implies,
\[E[X_t] \geq \frac{t+1}{p} - t\log(n) + t.\]
Let $D_t$ be the geometric random variable associated with $X_t$. Then we have $X_t \sim X_{t-1} - \log(n) +1 + D_t$, which is $X_t \sim \sum_{i=0}^{t} D_i - t\log(n) + t$. If we let $\delta = \frac{1}{2} + \frac{tp}{2(t+1)} - \frac{tp\log(n)}{2(t+1)}$, $a_t = \left(\frac{t+1}{p} - t\log(n) + t\right)/2$ and use Chernoff bounds for the sum of independent geometric random variables \cite[Theorem~1.10.32]{Doerr_2019} we have, 
\begin{align*}
P\left(X_t \leq a_t\right)
 &=  P\left(\sum_{i = 0}^{t} D_i \leq\frac{t+1}{2p}+\frac{t\log(n)}{2}-\frac{t}{2}\right) \\
&= P\left( \sum_{i = 0}^{t} D_i \leq (1-\delta)\frac{t+1}{p}\right) \\
& \leq \exp\left(-\frac{\delta^2(t+1)}{2 - \frac{4\delta}{3}}\right).
\end{align*}
Since $p \leq \frac{1}{2\log(n)}$, we have $\delta \geq \frac{1}{4}$. Therefore,
\begin{align*}
P\left(X_t \leq \left(\frac{t+1}{p} - t\log(n) + t\right)/2\right)
    & \leq \exp\left(-\frac{\delta^2(t+1)}{2 - \frac{4\delta}{3}}\right) \leq \exp\left(-\frac{3t}{80}\right).
\end{align*}
When $t = \log^2 (n)$, 
\begin{align*}
P\left(X_t \leq \left(\frac{t+1}{p} - t\log(n) + t\right)/2\right)
    &\leq \exp\left(-\frac{3t}{80}\right)\\
    &=n^{-\frac{3}{80}\log(n)}.
\end{align*}
Now assume that we sampled a search point with noise at least $ m = \frac{ t + 1}{2p} - \frac{t\log(n)}{2} + \frac{t}{2}$, where $t = \log^2 (n)$. As we have seen at any given iteration the probability that the standard mutation operator flips more than $ck\log(n) - 1$ bits is very low, we will again analyze the worst case scenario. For a neighbor of the current search point with at most $ck\log(n) - 1$ to have higher fitness it should have at least $m - ck\log(n)$ noise. The probability for this to happen is, 
\begin{align*}
P(D \geq m - ck\log n + 1)  = p(1-p)^{m - ck\log(n)}
 \leq e^{-p(m-ck\log(n))} = e^{\frac{ck}{2}} n^{-\frac{1}{4}\log n}
\end{align*}
For a given $k$, let $D_{m_{1}}, \dots , D_{m_{n^k}}$ denote the random geometric noise associated with $n^k$ neighbors of the current search point with at least noise $m$. Then probability that within $n^k$ steps at least one of the neighbours will be of higher fitness is at most
\begin{align*}
P\left(\bigcup_{i = 1}^{n}(D_{m_{i}} \geq m - ck\log(n) + 1)\right)
    &\leq \sum_{i=1}^{n^k}  P(D_{m_{i}} \geq m - ck\log(n) + 1)\\
    & \leq e^{\frac{ck}{2}}n^{k-\frac{\log(n)}{4}}.
\end{align*}
\end{proof}
\qed

\section{Performance of Random Search}
\label{sec:randomSearch}

For comparison with the performance of RLS and the \oneoneEA, we briefly consider Random Search (RS) in this section. We state the theorem that Random Search has a bound of $O(\sqrt{n\log n})$ for the number of $1$s found within a polynomial number of iterations and can be proved by Chernoff bounds.

\begin{thm}\label{thm:RSonOM}
Let $c >0$ be given and $t \geq 1$. Then the bit string with the most number of $1$s found by Random Search within $t \leq n^k$ iterations, choosing a uniformly random bit string each iteration, has at most $n/2 + O(\sqrt{n \log n})$ $1$s with probability $1-O(n^{-c})$.
\end{thm}

\section{Experimental Evaluation}\label{sec:ExperimentsonOneMax}

In this section we empirically analyze the performance of the cGA, the RLS, the \oneoneEA and the Random Search algorithms on the rugged \OneMax with two different noise models. We considered noise sampled from the normal distribution with mean zero and variance 5 and another noise model sampled from the geometric distribution with variance 5. From the results we can see that after $n^2$ iterations, where $n$ is length of the bit string, the RLS and the \oneoneEA  does not sample a search point with more than $60\%$ of $1$s but the cGA with $K = \sqrt{n}\log(n)$ always finds the optimum. The plot in Figure~\ref{fig:rls_ooea} is mean of 100 independent runs of each algorithm for bit string lengths $100$ to $1000$ with step size $100$. To have closer look at the performance of the other algorithms, performance of the  cGA(straight line at 100) is removed from the plot.

%
\begin{figure}[htbp]
  \centering
    \includegraphics[width=0.7\columnwidth]{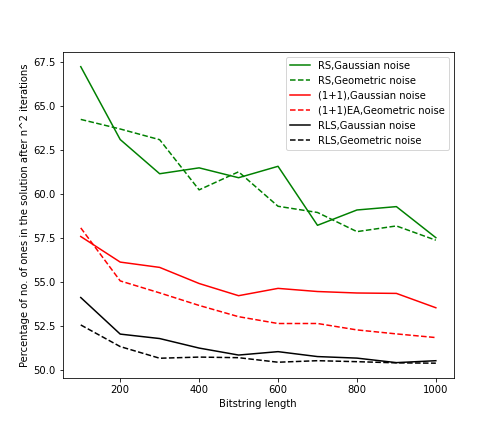}
    \caption{Percentage of ones in the sampled search point after $n^2$ iterations of Random Search, \oneoneEA and RLS on optimizing rugged OneMax function with Normal and Geometric noise both having variance of $5$. Performance of cGA which is a straight line at 100\% is omitted.}
    \label{fig:rls_ooea}
\end{figure}
%
Each time the first search point was set to have  $50\%$ of $1$s. As $n$ increases the percentage of ones in the search point sampled by RLS and \oneoneEA after $n^2$ iteration tends to $50\%$ and it tends to less than $60\%$ in case of random search.

\section{Conclusion}\label{sec:conclusion}

Rugged fitness landscapes appear in many real-world problems and may lead to algorithms getting stuck in local optima. We investigated this problem in this paper for the rugged \OneMax problem which is obtained through a noisy \OneMax fitness function where each search point is only evaluated once. We have shown that RLS and the (1+1)~EA can only achieve small improvements on an initial solution chosen uniformly at random. In contrast to this the cGA is able to improve solution quality significantly until it is (almost) optimal. Our experimental investigations show this behaviour for realistic input sizes and also point out that RLS and the \oneoneEA perform significantly worse than random search on the rugged \OneMax problem.

\section*{Acknowledgements}
This work was supported by the Australian Research Council through grant FT200100536.

\bibliographystyle{splncs04}
%
\bibliography{main}





\appendix
\section{Important Lemmas}\label{sec:lemmas}
In the following lemma we will show that the probability of choosing a point from a normal distribution with mean greater than zero which will be less than minimum of $n$ independently distributed points from $\mathcal{N}(0, 1)$ is $\Omega(1/n)$.
\begin{lem}\label{lem:minimumOfGaussians}
Let $c > 0$ and, for all $i \leq n$, $Z_i \sim \mathcal{N}(0, 1)$ and $Y_c \sim \mathcal{N}(c, 1)$ all independently distributed. Let $X_n = \min_{i \leq n} Z_i$($n$th order statistic of $z_i$'s\cite{David2011}). Then we have that, for all $c \in \realnum_+$ and all $n$, 
$$
P(Y_{c} < X_n) > \frac{1}{n+1}\left(1 - c\left(1 + \sqrt{2\ln(n+1)} +  \sqrt{2\ln(1/c)} + \frac{c}{2}\right)\right).
$$
\end{lem}


\begin{proof}


We start by computing the probability that the random variable $Z_i$ is at least a given value $x$. To this end, let $i \leq n$ and $x \in \realnum$. We have
$$
 P(Z_i > x) = 1 - P(Z_i \leq x) = 1 - \frac{1}{2} - \frac{1}{2}\erf\left(\frac{x}{\sqrt{2}}\right) = \frac{1}{2} \erfc\left(\frac{x}{\sqrt{2}}\right).
$$
 
Now let $f$ be the probability density function of $Y_{c}$ and let $a = -\sqrt{2}(\sqrt{\ln(n+1)} +  \sqrt{\ln(1/c)})$. We have

\begin{align*}
 P(Y < X_n) &= \int_{-\infty}^{\infty} f(x) P(X_n > x)\,dx \\
 &= \int_{-\infty}^{\infty} f(x) P(\bigcap_{i =1}^n (Z_i > x))\,dx \\ 
 &= \int_{-\infty}^{\infty} f(x)  \prod_{i=1}^{n}P(Z_i > x)\,dx \\
 &= \int_{-\infty}^{\infty} \frac{1}{\sqrt{2\pi}} e^{-\frac{(x-c)^2}{2}} \left(\frac{1}{2}  \erfc\left(\frac{x}{\sqrt{2}}\right)\right)^n\,dx\\
&= \int_{-\infty}^{\infty} \frac{1}{\sqrt{2\pi}} e^{-\frac{x^2}{2}} e^{-\frac{c^2}{2}} e^{cx}\left(\frac{1}{2}  \erfc\left(\frac{x}{\sqrt{2}}\right)\right)^n\,dx\\
&> e^{-\frac{c^2}{2}} e^{ca}\int_{a}^{\infty} \frac{1}{\sqrt{2\pi}} e^{-\frac{x^2}{2}} \left(\frac{1}{2}  \erfc\left(\frac{x}{\sqrt{2}}\right)\right)^n\,dx.\\
\end{align*}

We continue by analyzing the integral and the factor before the integral separately, starting with the factor. We use the value of  $a$ 
and get
\begin{align*}
e^{c(a-\frac{c}{2})} 
    &\geq 1 + c\left(a-\frac{c}{2}\right) \\
    &\geq 1 - c\left(\sqrt{2\ln(n+1)} +  \sqrt{2\ln(1/c)} + \frac{c}{2}\right).
\end{align*}
  Now we turn to bounding the integral. Lets integrate over the function $g(x) = \frac{1}{\sqrt{2\pi}} e^{-\frac{x^2}{2}} \left(\frac{1}{2}  \erfc\left(\frac{x}{\sqrt{2}}\right)\right)^n$
  . By symmetry, the probability of a $\normal{0,1}$-distributed random variable being smaller than $X_n$ is $1/(n+1)$. Thus, we get
\begin{align*}
\int_{a}^\infty g(x) dx 
    &= \int_{-\infty}^{\infty} g(x) \,dx - \int_{-\infty}^a g(x) dx\\
    & = \frac{1}{n+1} - \int_{-\infty}^a g(x) dx\\
    & \geq \frac{1}{n+1} - \int_{-\infty}^a \exp(-x^2/2)/\sqrt{2\pi} dx\\
    & =  \frac{1}{n+1} - \frac{1}{2}(1+\erf(a/\sqrt{2}))\\
    & =  \frac{1}{n+1} - \frac{1}{2}(2-\erfc(a/\sqrt{2}))\\
    & =  \frac{1}{n+1} - \frac{1}{2}(\erfc(-a/\sqrt{2}))\\
\shortintertext{We now use that, for $x>0$, $\erfc(x) \leq \frac{2}{\sqrt{\pi}}\cdot \frac{e^{-x^2}}{x+\sqrt{x^2+\frac{4}{\pi}}}$\cite{WAUB}.} 
& \geq \frac{1}{n+1} - \frac{1}{2}\left(\frac{2}{\sqrt{\pi}} \frac{\exp(-a^2/2)}{-2a/\sqrt{2}}\right)\\
    & = \frac{1}{n+1} - \left(\frac{1}{\sqrt{2\pi}} \frac{\exp(-a^2/2)}{|a|}\right)\\
    & \geq \frac{1}{n+1} - \frac{\exp(-a^2/2)}{|a|}\\
    & \geq \frac{1}{n+1} - \exp(-a^2/2)\\
\shortintertext{We continue using $a^2 \geq 2\ln(n+1) + 2\ln(1/c)$.}
& \geq \frac{1}{n+1} - \exp(-\ln(n+1) - \ln(1/c))\\
    & = \frac{1}{n+1}\left(1 - c\right).
\end{align*}
Overall, we get
\begin{align*}
P(Y < X_n) &> \frac{1-c}{n+1}\left(1 - c\left(\sqrt{2\ln(n+1)} +  \sqrt{2\ln(1/c)} + \frac{c}{2}\right)\right)\\
& \geq \frac{1}{n+1}\left(1 - c\left(1 + \sqrt{2\ln(n+1)} +  \sqrt{2\ln(1/c)} + \frac{c}{2}\right)\right).
\qed
\end{align*}

\end{proof}
\qed

\begin{lem}\label{lem:movingWeight}
Let $\varepsilon \in (0,1)$ and define
$$
M_\varepsilon = \set{p \in [0.25,1]^n}{ \sum_{i=1}^n p_i \leq n(1-\varepsilon)}.
$$
Then
$$
\max_{p,q \in M_\varepsilon} \sum_{i=1}^n (2p_iq_i-p_i-q_i) \leq - n \varepsilon/2.
$$
\end{lem}

\begin{proof}
Let $p,q \in [0.25,1]$ and define $\overline{p} = \sum_{i=1}^n p_i$ and $\overline{q} = \sum_{i=1}^n q_i$. We have that
$$
\sum_{i=1}^n (2p_iq_i-p_i-q_i) = -\overline{p}-\overline{q}+2\sum_{i=1}^n p_iq_i.
$$
Keeping $\overline{p}$ and $\overline{q}$, this latter term is maximized by maximizing $\sum_{i=1}^n p_iq_i$. This is obtained by moving $p$-weight from summands with smaller co-factors to summands with larger co-factors, and similarly for $q$-weight. Thus, the term is maximized by supposing as many $p$- and $q$-weights at the upper border as possible, and the others at $0.25$.

For simplicity, we suppose that there are exactly $k$ at the lower border of $0.25$ and $n-k$ at the upper border for $p$, and $\ell$ at the lower border and $n-\ell$ at the upper border for $q$.
Therefore, $n-k + k/4 = \overline{p}$ and $n-\ell + \ell/4 = \overline{q}$, so we derive $k = 4(n-\overline{p})/3$ and $\ell = 4(n-\overline{q})/3$.

Without loss of generality, we suppose $k \geq \ell$; this gives $\overline{q} \geq \overline{p}$. Then we have
\begin{align*}
\sum_{i=1}^n p_iq_i & \leq (n- k) \cdot 1 + (k - \ell)/4 + \ell/16\\
    &= n- \frac{4(n-\overline{p})}{3}  + \frac{4(n-\overline{p})/3 - 4(n-\overline{q})/3}{4} + \frac{4(n-\overline{q})}{(3\cdot 16)}\\
    &= -n/3 + 4\overline{p}/3  + (\overline{q}-\overline{p})/3 + (n-\overline{q})/12\\
    &= -n/4 + \overline{p} + \overline{q}/4.
\end{align*}

Thus, we have
\begin{align*}
 -\overline{p}& -\overline{q}+2\sum_{i=1}^n p_iq_i \\
 &=   -\overline{p}-\overline{q}+2(-n/4 + \overline{p} + \overline{q}/4)\\
 &=   -n/2 + \overline{p} - \overline{q}/2\\
 &\leq   -n/2 + \overline{q} - \overline{q}/2\\
 &=   -(n-\overline{q})/2\\
 &\leq   -n\varepsilon/2.
\end{align*}
\end{proof}
\qed

\end{document}